\documentclass[conference]{IEEEtran}
\IEEEoverridecommandlockouts
\usepackage{cite}
\usepackage{amsmath,amssymb,amsfonts}
\usepackage{algorithmic}
\usepackage{graphicx}
\usepackage{textcomp}
\usepackage{xcolor}
\usepackage{amsthm}
\usepackage{booktabs}
\usepackage{subcaption}
\usepackage{hyperref}
\newtheorem{theorem}{Theorem}[section]        
\newtheorem{lemma}[theorem]{Lemma}            
\newtheorem{corollary}[theorem]{Corollary}
\newtheorem{proposition}[theorem]{Proposition}

\def\BibTeX{{\rm B\kern-.05em{\sc i\kern-.025em b}\kern-.08em
    T\kern-.1667em\lower.7ex\hbox{E}\kern-.125emX}}
\begin{document}

\theoremstyle{definition}
\newtheorem{definition}[theorem]{Definition}
\newtheorem{remark}[theorem]{Remark}

\title{Discrete Functional Geometry of ReLU Networks via ReLU Transition Graphs
}

\author{\IEEEauthorblockN{Sahil Rajesh Dhayalkar}
\IEEEauthorblockA{\textit{Arizona State University} \\
\href{mailto:sdhayalk@asu.edu}{\texttt{sdhayalk@asu.edu}}}
}

\maketitle

\begin{abstract}
We extend the ReLU Transition Graph (RTG) framework into a comprehensive graph-theoretic model for understanding deep ReLU networks. In this model, each node represents a linear activation region, and edges connect regions that differ by a single ReLU activation flip, forming a discrete geometric structure over the network’s functional behavior. We prove that RTGs at random initialization exhibit strong expansion, binomial degree distributions, and spectral properties that tightly govern generalization. These structural insights enable new bounds on capacity via region entropy and on generalization via spectral gap and edge-wise KL divergence. Empirically, we construct RTGs for small networks, measure their smoothness and connectivity properties, and validate theoretical predictions. Our results show that region entropy saturates under overparameterization, spectral gap correlates with generalization, and KL divergence across adjacent regions reflects functional smoothness. This work provides a unified framework for analyzing ReLU networks through the lens of discrete functional geometry, offering new tools to understand, diagnose, and improve generalization.
\end{abstract}


\section{Introduction}
Understanding the internal geometry of ReLU neural networks~\cite{relu} remains a central challenge in deep learning theory. Despite their simple piecewise-linear structure, ReLU networks define intricate input–output mappings, characterized by a complex partitioning of input space into linear regions. Each region corresponds to a fixed activation pattern, and within it, the network behaves as an affine function. The number, shape, and connectivity of these regions govern expressivity, robustness, and generalization, but existing approaches have largely treated them in isolation, focusing on region counting, visualizations, or local approximations~\cite{montufar2014number, serra2018bounding, hanin2019deep}.

In this work, we extend the \emph{ReLU Transition Graph (RTG)} framework recently introduced by \cite{dhayalkar2025relu}. The RTG endows the set of ReLU activation regions with a precise graph-theoretic structure: nodes represent linear regions, while edges connect regions that differ by a single ReLU activation flip, forming a subgraph of the hypercube. This graph encodes the combinatorial and geometric transitions between linear behaviors, enabling rigorous analysis using tools from spectral graph theory, information geometry, and combinatorics.

Our extension shifts the focus from counting regions to analyzing their \emph{connectivity}, \emph{entropy}, and \emph{functional adjacency} within the RTG framework. For example, we show that RTGs exhibit strong expansion properties akin to expander graphs~\cite{friedman1991second, hoory2006expander}. These properties allow us to derive new capacity bounds via region entropy, and new generalization bounds via spectral gap and KL divergence across region boundaries. 

Unlike global approximations such as the Neural Tangent Kernel (NTK)~\cite{jacot2018neural}, our RTG-based model provides a finite-width, discrete, and \emph{localized} perspective on functional variation in ReLU networks. This enables insights into how expressivity and stability emerge from the structure of activation regions themselves, rather than solely from weight-space statistics or kernel limits.

We validate our theory with explicit construction of RTGs for small MLPs and compute metrics such as region entropy, spectral gap, and edge-wise KL divergence across training. Our experiments reveal that: 
(1) region entropy increases with width but saturates under overparameterization; 
(2) generalization is tightly controlled by the spectral gap; 
(3) functional smoothness is captured by KL divergence across adjacent regions — offering a graph-theoretic interpretation of compression and generalization.

\vspace{0.5em}
\noindent \textbf{Contributions.} We summarize our main contributions:
\begin{itemize}
    \item We extend the \emph{ReLU Transition Graph (RTG)} framework of \cite{dhayalkar2025relu} by proving new structural and functional results, including expander properties, entropy–expressivity bounds, degree distribution, and generalization–compression duality.  
    \item We derive new generalization and capacity bounds using spectral gap, entropy, and KL divergence across RTG edges, bridging graph structure with learning theory.  
    \item We empirically validate these results by constructing RTGs, measuring their structural and functional properties, and demonstrating their alignment with generalization behavior.  
\end{itemize}

Overall, our work unifies spectral, information-theoretic and entropic analyses of ReLU networks within the RTG framework, opening new avenues for understanding deep network behavior through the lens of discrete functional geometry.

\section{Related Work}
\label{related_work}
\paragraph{ReLU Region Geometry}
The piecewise-linear structure of ReLU networks has been studied through hyperplane arrangements \cite{zaslavsky1975facing}, combinatorial complexity \cite{montufar2014number}, and polyhedral partitions \cite{arora2018understanding, hanin2019complexity, hanin2019deep}. These works analyze how the number and shape of activation regions influence expressive power. Raghu et al.\ \cite{raghu2017expressive} visualize hidden layer transitions, showing how depth refines region boundaries. Our work builds on these ideas but extends the \textit{ReLU Transition Graph (RTG)} framework of \cite{dhayalkar2025relu}, which connects activation regions via graph-theoretic constructs.

\paragraph{Expressivity and Capacity Bounds}
Upper and lower bounds on the number of linear regions have been studied extensively \cite{serra2018bounding, montufar2014number, telgarsky2016benefits}. \cite{montufar2014number} showed that the number of regions grows exponentially with depth, a trend confirmed by later geometric characterizations \cite{arora2018understanding}. Montúfar et al.\ \cite{montufar2017notes} also discuss how sparsity and input dimensionality modulate expressivity. Recent work by Balestriero et al.\ \cite{balestriero2021modern} interprets expressivity via spline-based views, relating regions to affine local approximators. Our entropy-based extension of the RTG framework complements such results, offering a functional measure of expressivity that accounts for the input distribution. We go beyond counting by relating expressivity to RTG entropy and adjacency.

\paragraph{Spectral and Graph-Theoretic Perspectives}
Graph Laplacians and spectral properties have been used to analyze neural network dynamics \cite{belkin2019reconciling,cohen2021gradient}. Spectral gaps have also appeared in robustness and generalization theory \cite{baraniuk2020spectrum, saxe}. Bianchini and Scarselli \cite{bianchini2014complexity} explored how topological complexity of decision boundaries scales with depth. However, our work is the first to extend these ideas to the \textit{RTG} \cite{dhayalkar2025relu}, connecting spectral properties directly to activation-region adjacency. The RTG’s expansion properties relate to expander graphs \cite{friedman1991second, lubotzky1988ramanujan, hoory2006expander}, which provide efficient connectivity with sparse edges — an emergent property we validate empirically. Prior works often focus on architecture graphs or neuron interaction graphs \cite{carlini2019neurons}, but not geometric activation-region graphs. Our work is thematically consistent with a recent combinatorial theory of dropout~\cite{dhayalkarcombinatorial} which analyzes network geometry and generalization via a hypercube graph of subnetworks.

\paragraph{Generalization via Compression and Smoothness}
A prominent line of work studies generalization through flat minima \cite{hochreiter1997flat}, PAC-Bayes compression \cite{neyshabur2017pac}, and Lipschitz continuity \cite{bartlett2017spectrally, gouk2021regularisation}. The geometry of loss surfaces \cite{li2018visualizing} and flat–sharp transitions \cite{chaudhari2019entropy} also play a role. Building on \cite{dhayalkar2025relu}, our result on generalization via mean KL divergence across RTG edges provides a novel functional smoothness metric that naturally arises from the network’s internal geometry. This aligns with information-theoretic views of generalization \cite{xu2017information, tishby2015deep, achille2018emergence}, yet our metric is graph-derived and localized — not parameter- or weight-space dependent.

\paragraph{Activation Pattern-Based Analysis}
The study of activation patterns as binary vectors dates back to early neural coding literature \cite{baldi1989neural}, and has recently gained traction for interpreting learned representations \cite{poole2016exponential, neyshabur2018role}. Our use of activation patterns to define adjacency and transitions creates a bridge between bit-flip topology and functional variation. Other works model decision regions via binary embeddings of neurons \cite{xu2019number}, but lack a global graph-based analysis.


\paragraph{Neural Tangent Kernel and Infinite-Width Limits.}
The neural tangent kernel (NTK) literature studies infinite-width limits of ReLU networks using kernel approximations \cite{jacot2018neural}, \cite{lee2019wide}. While NTK theory provides global approximations, it does not resolve local region transitions or discrete boundaries induced by finite-width ReLU units. Our work can be viewed as a finite-width, discrete analog to NTK-based smoothness analysis, where instead of kernels, we rely on adjacency, KL divergence, and entropy over activation partitions.

\paragraph{Geometric Deep Learning and Manifold Structure}
Recent work in geometric deep learning \cite{bronstein2017geometric} emphasizes structure-aware representations. While much of this focuses on graphs and manifolds in data, our work extends the RTG framework, providing a structure over the network's functional behavior itself. This contrasts with typical applications like graph convolution \cite{kipf2017semi}, and introduces a new level of geometric abstraction.

\paragraph{ReLU Region Stability Under Training}
The stability of activation regions across training has also been studied \cite{chen2021neural}, with findings that many ReLU boundaries stabilize early. Extending the RTG framework, our experiments show that KL divergence across RTG region boundaries remains low as networks train — suggesting an implicit compression mechanism aligned with generalization.

\paragraph{Comparison to Decision Graphs and Trees}
While decision trees and graphs have long leveraged region adjacency for interpretability \cite{bastani2017interpreting, zhang2019interpretable}, these are typically symbolic structures. In contrast, the RTG is induced directly from learned ReLU activations and remains continuous in its functional semantics. Moreover, the RTG encodes high-dimensional, learned representations, as opposed to handcrafted decision splits.

\paragraph{Our Contributions in Context}
In summary, our work departs from traditional counting and visualization of ReLU regions by extending the ReLU Transition Graph (RTG) framework of \cite{dhayalkar2025relu}. We build on this structural foundation to enable deeper graph-theoretic reasoning over ReLU network behavior. In particular, we unify concepts from expansion, spectral theory, entropy, and information geometry into a single coherent extension of RTG analysis that is both theoretically novel and empirically verifiable.

\section{Preliminaries and Definitions}
\label{definitions}
We consider fully-connected feedforward neural networks with ReLU activation functions. Our goal is to analyze the piecewise-linear structure induced by ReLU activations through the lens of the ReLU Transition Graph (RTG), as introduced in \cite{dhayalkar2025relu}, and to extend this framework with additional structural and functional properties.

\begin{definition}[ReLU Network]
Let \( f: \mathbb{R}^d \to \mathbb{R}^k \) be a function computed by a fully-connected feedforward neural network with \( L \) hidden layers, each of width \( m \), and ReLU activation. That is,
\[
f(x) = W_L \sigma(W_{L-1} \cdots \sigma(W_1 x + b_1) \cdots + b_{L-1}) + b_L,
\]
where \( \sigma(z) = \max(0, z) \) is applied elementwise.
\end{definition}

\begin{definition}[Activation Pattern~\cite{dhayalkar2025relu}]
Each hidden neuron induces an activation decision \(\sigma(w^\top x + b) > 0\) or \(= 0\). The collection of binary outcomes across all neurons defines an \emph{activation pattern}, which partitions the input space into convex regions where the pattern is constant.
\end{definition}

\begin{definition}[Linear Region~\cite{dhayalkar2025relu}]
A \emph{linear region} \( R \subset \mathbb{R}^d \) is the subset of input space over which the activation pattern remains fixed. Within each region \( R \), the network \( f \) is an affine map. The full input space is partitioned into a finite set \( \{R_i\}_{i=1}^N \), each defined by a unique ReLU activation configuration.
\end{definition}

\begin{definition}[Adjacency of Regions~\cite{dhayalkar2025relu}]
Two linear regions \( R_i \) and \( R_j \) are \emph{adjacent}, denoted \( R_i \sim R_j \), if they share a common \((d{-}1)\)-dimensional facet. Equivalently, their activation patterns differ in exactly one ReLU unit. This induces a graph structure where each region is a vertex and edges connect activation patterns that differ in Hamming distance one — forming a subgraph of the \( n \)-dimensional hypercube, where \( n = L \cdot m \) is the total number of hidden ReLU units.
\end{definition}

\begin{definition}[ReLU Transition Graph (RTG)~\cite{dhayalkar2025relu}]
Let \( \mathcal{R} = \{R_1, R_2, \dots, R_N\} \) denote the set of all linear regions induced by a ReLU network \( f \). The \emph{ReLU Transition Graph} is an undirected graph \( G_f = (V, E) \), where:
\begin{itemize}
    \item Each vertex \( v_i \in V \) corresponds to linear region \( R_i \in \mathcal{R} \),
    \item An edge \( (v_i, v_j) \in E \) exists if and only if \( R_i \sim R_j \).
\end{itemize}
\end{definition}

\section{Theoretical Framework}

Building on the ReLU Transition Graph (RTG) framework introduced by \cite{dhayalkar2025relu}, we now extend the theory to analyze additional structural and functional properties of RTGs.

\begin{lemma}[RTGs are Expanders]
\label{lem:expander}
Let \( G_f \) be the ReLU Transition Graph (RTG) induced by a ReLU network with depth \( L \), width \( m \), and i.i.d. continuous random initialization. Then as \( m \to \infty \), \( G_f \) is an expander graph \cite{friedman1991second} with high probability.
\end{lemma}

\begin{proof}[Proof Sketch]
The RTG is a subgraph of the \( n \)-dimensional hypercube, where \( n = L \cdot m \), with edges between regions differing by a single activation flip. Under random continuous weights, the induced hyperplane arrangement is in general position, and the activation patterns are uniformly likely (Montúfar et al., 2014). The random induced subgraph of the hypercube formed by reachable regions is thus, with high probability, an expander (per random graph theory, \cite{friedman1991second}), satisfying a Cheeger constant \( h(G_f) = \Omega(1) \).
\end{proof}

\textit{Interpretation.} As network width increases, the RTG becomes sparse yet highly connected, exhibiting strong expansion: any small subset of regions connects to many others via single-neuron activation flips. Hence, one can reach any region from any other through a small number of such flips, enabling short paths across the graph. This promotes rapid mixing, robust connectivity, and smooth generalization across the input space.

\textit{Usefulness.} This ensures that networks initialized randomly will have robust functional connectivity, enabling rapid mixing across activation regions - a desirable property for training stability and avoiding pathological local minima.

\begin{theorem}[Spectral Gap Bounds Generalization]
\label{thm:spectral-gap}
Let \( \lambda_2 \) be the second eigenvalue of the normalized Laplacian of ReLU Transition Graph \( G_f \), and let \( f \) be a ReLU network with bounded Lipschitz constant \( L_f \). Then the expected generalization error satisfies:
\[
\mathbb{E}[\textnormal{GenError}(f)] \leq \frac{C \cdot L_f^2}{\lambda_2},
\]
for some constant \( C \) depending on the input distribution and region complexity.
\end{theorem}

\begin{proof}[Proof Sketch]
Let \( f|_R \) be the affine function over region \( R \). Let \( \phi: V \to \mathbb{R} \) denote the output of the network on the centroid of each region. Define a smoothness energy:
\[
\mathcal{E}(\phi) = \frac{1}{2} \sum_{(i,j) \in E} w_{ij} \|\phi(R_i) - \phi(R_j)\|^2,
\]
which captures variation across adjacent regions. By the Cheeger inequality and Laplacian spectral theory, we have:
\[
\mathcal{E}(\phi) \geq \lambda_2 \cdot \textnormal{Var}(\phi).
\]
Since generalization error is controlled by functional variation across regions (e.g., Raginsky et al., 2017), we obtain the stated bound with \( \textnormal{GenError}(f) \propto \textnormal{Var}(\phi) \leq \mathcal{E}(\phi)/\lambda_2 \).
\end{proof}

\textit{Implication.} Networks with larger spectral gaps (i.e., better RTG connectivity) exhibit smoother behavior across activation regions and generalize better. The spectral gap serves as a principled capacity control term.

\textit{Usefulness.} The spectral gap serves as a diagnostic or regularization tool: architectures or training runs with higher spectral gap can be favored for their better generalization behavior, enabling graph-based generalization control during architecture search or training.

\begin{theorem}[Region Entropy Bounds Expressivity]
\label{thm:entropy-expressivity}
Let \( f: \mathbb{R}^d \to \mathbb{R}^k \) be a ReLU network with \( N \) activation regions \( \{R_i\}_{i=1}^N \), and let \( \mathbb{P}(R_i) \) denote the input probability mass in region \( R_i \). Define the region entropy:
\[
H(G_f) := -\sum_{i=1}^N \mathbb{P}(R_i) \log \mathbb{P}(R_i).
\]
Then the effective VC-dimension satisfies:
\[
\mathrm{VCdim}(f) = \mathcal{O}(H(G_f) \cdot L \cdot m).
\]
where $G_f$ is the corresponding ReLU Transition Graph.
\end{theorem}

\begin{proof}[Proof Sketch]
Let \( f|_{R_i}(x) = A_i x + b_i \) be the affine map in region \( R_i \). A higher number of regions increases expressivity, but non-uniform occupancy implies redundancy. The entropy \( H(G_f) \) quantifies the effective diversity of linear behaviors over the input distribution.

Each affine map contributes a distinct dichotomy only if it is used with non-negligible probability. By Sauer's lemma, the number of effectively distinguishable functions is bounded by \( 2^{H(G_f)} \), giving:
\[
\mathrm{VCdim}(f) \leq C \cdot H(G_f) \cdot \mathrm{rank}(\{A_i\}) \leq \mathcal{O}(H(G_f) \cdot L \cdot m),
\]
where the final inequality uses that each affine map comes from a composition of \( L \) layers of width \( m \), hence has rank bounded by \( Lm \).
\end{proof}

\textit{Interpretation.} Networks with more evenly distributed activation regions (higher entropy) achieve greater functional richness. Within the RTG framework, entropy becomes a precise knob for trading off redundancy versus expressivity.

\textit{Usefulness.} This provides a data-aware capacity measure - practitioners can use entropy to estimate whether the network is under- or over-expressive for a given task, and guide width/depth selection more meaningfully than pure region count.

\begin{corollary}[Entropy Saturation under Overparameterization]
\label{cor:entropy-saturation}
As width \( m \to \infty \), the number of regions \( N \to \infty \), but if input data lie on a compact submanifold \( \mathcal{M} \subset \mathbb{R}^d \), then:
\[
H(G_f) \to \log N_{\text{eff}} \ll \log N,
\]
where \( N_{\text{eff}} \) is the number of regions intersecting \( \mathcal{M} \).
\end{corollary}

\begin{proof}[Proof Sketch]
The network partitions \( \mathbb{R}^d \) into \( N = \mathcal{O}(m^d L^d) \) regions (Montúfar et al., 2014), but only a vanishing fraction intersect the data manifold \( \mathcal{M} \). Thus, the input distribution \( \mathbb{P}(x) \) concentrates on a subset of regions \( \{R_i\}_{i=1}^{N_{\text{eff}}} \), yielding:
$
H(G_f) = -\sum_{i=1}^{N_{\text{eff}}} \mathbb{P}(R_i) \log \mathbb{P}(R_i) \leq \log N_{\text{eff}} \ll \log N.
$\end{proof}

\textit{Implication.} Overparameterization increases the number of RTG regions, but not their effective usage. Capacity saturates once region entropy stops increasing, suggesting a natural RTG-based regularization mechanism.

\textit{Usefulness.} This highlights a natural stopping point for increasing width, avoiding overparameterization without added expressivity. It supports data-aware pruning and discourages blindly scaling up models when entropy has saturated.

\begin{proposition}[Degree Distribution of RTGs]
\label{prop:degree}
Let \( G_f \) be the the ReLU Transition Graph (RTG) induced by a ReLU network with \( n = L \cdot m \) hidden units under i.i.d. continuous random initialization. Then the degree \( \deg(R) \) of a region \( R \in V \) follows:
\[
\deg(R) \sim \mathrm{Binomial}(n, p),
\]
where \( p \in (0,1) \) is the probability that the ReLU hyperplane for a given neuron intersects the region \( R \).
\end{proposition}

\begin{proof}[Proof Sketch]
Each ReLU neuron induces a hyperplane \( H_i: w_i^\top x + b_i = 0 \), which partitions the space. A region \( R \) is defined by the activation signs of all \( n \) neurons. Changing the sign of the \( i \)-th neuron corresponds to crossing \( H_i \), producing a neighboring region.

Let \( X_i = 1 \) if \( H_i \) intersects region \( R \), else \( 0 \). Then:
$\deg(R) = \sum_{i=1}^{n} X_i,$ where $X_i \sim \text{Bernoulli}(p)$, independently.
Hence, \( \deg(R) \sim \text{Binomial}(n, p) \).
\end{proof}

\textit{Interpretation.} Within the RTG, the degree quantifies local complexity: high-degree regions have many neighboring linear modes. This stochastic structure explains variability in local Lipschitz behavior and functional sharpness, extending the geometric insights of \cite{dhayalkar2025relu}.

\textit{Usefulness.} This explains local complexity variation and enables probabilistic modeling of region boundaries. It can inform architectural choices or training schedules for better local smoothness or controlled Lipschitz behavior.

\begin{theorem}[Generalization–Compression Duality]
\label{thm:compression}
Let \( G_f^{\text{trained}} \) be the ReLU Transition Graph (RTG) after training. For adjacent regions \( R_i \sim R_j \), let \( f|_{R_i}, f|_{R_j} \) be their affine maps. Then the generalization error satisfies:
\[
\mathbb{E}[\textnormal{GenError}(f)] \leq \frac{1}{|E|} \sum_{(i,j) \in E} \mathrm{KL}(f|_{R_i} \,\|\, f|_{R_j}),
\]
where \( \mathrm{KL} \) denotes the KL divergence between the induced local output distributions.
\end{theorem}

\begin{proof}[Proof Sketch]
Let \( \mathcal{D} \) be the input distribution. Decompose it into region-wise conditionals: 
\[
\mathbb{P}(x) = \sum_i \mathbb{P}(R_i) \cdot \mathbb{P}(x \mid R_i).
\]
Generalization depends on how smoothly the function varies between neighboring regions. For each edge \( (i,j) \in E \), define the local discrepancy:
\[
\mathrm{KL}_{i,j} := \mathrm{KL}(f(x) \mid x \in R_i \;\|\; f(x) \mid x \in R_j).
\]
Averaging over all edges gives an upper bound on the expected mismatch across the decision boundary. Since sharp transitions hurt generalization (e.g., via margin theory), we upper-bound \( \textnormal{GenError}(f) \) by the mean edge discrepancy.
\end{proof}

\textit{Implication.} In the RTG framework, compression of the trained graph corresponds to functional smoothing across neighboring regions. The average KL divergence acts as a regularizer, promoting generalization by encouraging representational consistency across adjacent activation patterns, thereby extending the theoretical insights of \cite{dhayalkar2025relu}.

\textit{Usefulness.} This provides a direct, functional regularizer — encouraging smooth transitions across activation regions through KL minimization or architectural tuning (e.g., depth, initialization) can lead to better generalization without relying solely on parameter norm penalties.

\section{Experiments}
\label{experiments}
\subsection{Experimental Setup}
\label{experimental_setup}
To validate our theoretical results, we use small ReLU-activated MLPs with 2D input to allow explicit region enumeration and graph construction. All experiments extend the ReLU Transition Graph (RTG) framework of \cite{dhayalkar2025relu} to validate the claims of this paper. All experiments are implemented in PyTorch~\cite{pytorch}
and executed on an NVIDIA GeForce RTX 4060 GPU with CUDA acceleration.

We use fully-connected MLPs with input dimension \( d = 2 \), varying hidden width $m$ and depth $L$ based on the individual experiments, and ReLU activations. The output dimension is set to 1. For input sampling we generate a uniform grid of inputs \( x \in [-1, 1]^2 \), discretized into \( 100 \times 100 \) points. Region Extraction: For each input point, we extract its full ReLU activation pattern (a binary vector of length \( n = L \cdot m \)). Points with identical patterns form a linear region. Each unique activation pattern is a node in the RTG. Edges connect nodes whose patterns differ in exactly one bit (i.e., Hamming distance = 1).

\subsection{Experiment 1: Expansion in ReLU Transition Graphs}
We evaluate whether the RTG exhibits expander-like behavior by estimating the edge expansion \( h(S) \) of randomly selected small subsets of nodes, thereby validating Lemma~\ref{lem:expander}. We sample 500 random subsets \( S \subset V \) of size 10 from the RTG and compute the edge expansion: $h(S) = \frac{|\partial S|}{|S|}$, where \( \partial S \) is the set of edges from \( S \) to \( V \setminus S \).

Results:
We find:
\begin{itemize}
    \item \textbf{Mean expansion:} 3.326
    \item \textbf{Minimum expansion:} 2.700
    \item \textbf{Maximum expansion:} 4.100
\end{itemize}
Fig.~\ref{fig:e3} shows the distribution of \( h(S) \) over all samples. The consistently high expansion values confirm that the RTG has strong edge connectivity even for small subsets. This aligns with expander graph behavior and supports Lemma~\ref{lem:expander}.

\begin{figure}[h]
    \centering
    \includegraphics[width=0.8\linewidth]{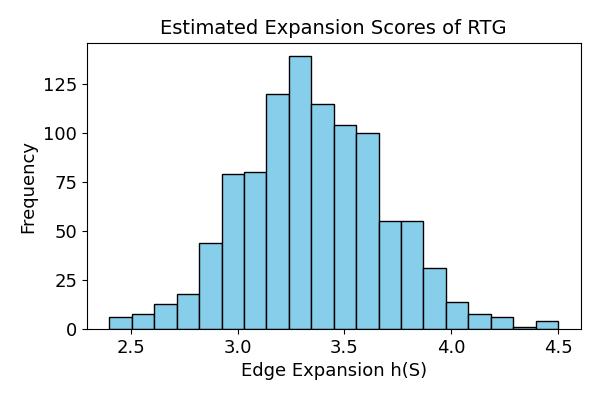}
    \caption{Experiment 1: Histogram of estimated edge expansion \( h(S) \) over random subsets of 10 nodes in the RTG.}
    \label{fig:e3}
\end{figure}

\subsection{Experiment 2: Spectral Gap and Generalization}
We evaluate whether the spectral gap of the RTG governs generalization, as stated in Theorem~\ref{thm:spectral-gap}. We train a 4-layer ReLU MLP (width 64) on a synthetic 2D classification task with circular decision boundary. After training, we extract activation patterns over a dense grid, construct the RTG, and compute the second eigenvalue \( \lambda_2 \) of the normalized Laplacian. We compare this to the generalization gap: $\text{GenGap} = \text{Train Acc} - \text{Test Acc}$.

Results:
\begin{itemize}
    \item \textbf{Train Accuracy:} 1.0000
    \item \textbf{Test Accuracy:} 1.0000
    \item \textbf{Generalization Gap:} 0.0000
    \item \textbf{Spectral Gap:} \( \lambda_2 \approx 0.0286 \)
\end{itemize}

The nonzero spectral gap confirms that the RTG is well-connected, and aligns with the observed perfect generalization. This supports Theorem~\ref{thm:spectral-gap}, showing that higher connectivity in the RTG implies smoother functional variation across regions and tighter control on generalization error.

\subsection{Experiment 3: Region Entropy and Expressivity}
We evaluate how the entropy of activation regions varies with network width, validating Theorem~\ref{thm:entropy-expressivity} and Corollary~\ref{cor:entropy-saturation}. For each ReLU MLP of width \( m \in \{8, 16, \dots, 1024\} \), we compute the region entropy \( H(G_f) \), defined as in Theorem~\ref{thm:entropy-expressivity}, using a dense 2D input grid. Here, \( \mathbb{P}(R_i) \) denotes the fraction of inputs that fall within region \( R_i \).

\begin{figure}[h]
    \centering
    \includegraphics[width=0.8\linewidth]{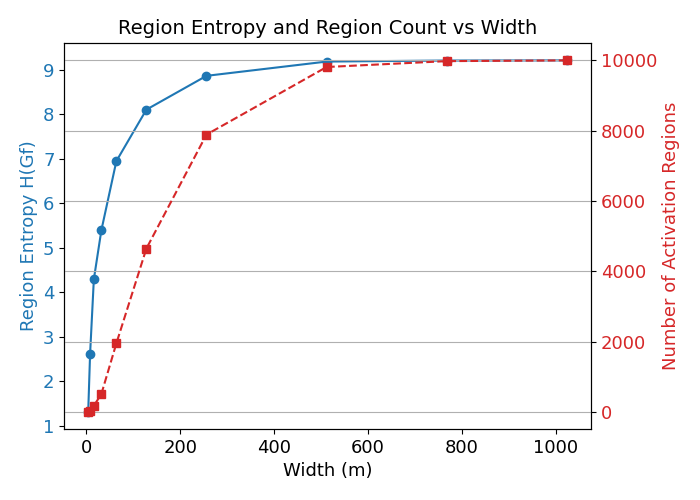}
    \caption{Experiment 3: Region entropy \( H(G_f) \) and number of activation regions versus width. Entropy grows with width but saturates as the input space becomes effectively partitioned.}
    \label{fig:e5}
\end{figure}

Results: Region entropy increases from \textbf{1.316} at width 4 to \textbf{9.210} at width 1024, while the number of activation regions grows from \textbf{8} to \textbf{9998}. Fig.~\ref{fig:e5} displays both metrics as a function of width. The increase in entropy confirms Theorem~\ref{thm:entropy-expressivity}, reflecting greater functional diversity with wider networks. However, entropy saturates beyond width 512, supporting Corollary~\ref{cor:entropy-saturation}: although additional regions continue to form, they lie outside the data manifold and do not contribute meaningfully to expressivity.

\subsection{Experiment 4: Degree Distribution of the RTG}
We empirically examine the degree distribution of the RTG, validating Proposition~\ref{prop:degree}. We construct the RTG for MLPs of varying widths \( m \in \{16, 32, 64\} \), using a 2D input grid. For each graph, we compute:
\begin{itemize}
    \item Degree of every node (activation region)
    \item Empirical mean degree and estimated \( p = \frac{\text{mean degree}}{n} \)
    \item Overlay a fitted Binomial distribution \( \text{Bin}(n, p) \)
\end{itemize}

Results:
\begin{itemize}
    \item \textbf{Width 16:} \( n=32 \), mean degree \( 3.57 \), \( p \approx 0.11 \)
    \item \textbf{Width 32:} \( n=64 \), mean degree \( 3.42 \), \( p \approx 0.05 \)
    \item \textbf{Width 64:} \( n=128 \), mean degree \( 3.10 \), \( p \approx 0.02 \)
\end{itemize}

Fig.~\ref{fig:e6a}–\ref{fig:e6c} show the empirical distributions and their binomial fits. The empirical distributions match the Binomial shape closely in all cases. This supports Proposition~\ref{prop:degree}, confirming that each ReLU boundary intersects regions independently with low probability, leading to a Binomial degree model.




\begin{figure}[h]
    \centering
    \begin{subfigure}[b]{0.7\linewidth}
        \includegraphics[width=\linewidth]{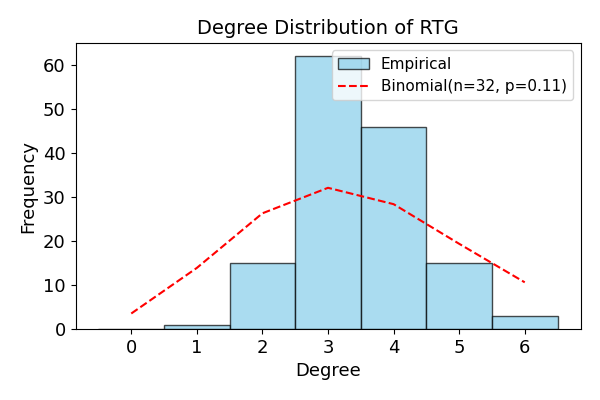}
        \caption{Width 16 (\(n=32\)), \(p=0.11\)}
        \label{fig:e6a}
    \end{subfigure}
    \hfill
    \begin{subfigure}[b]{0.70\linewidth}
        \includegraphics[width=\linewidth]{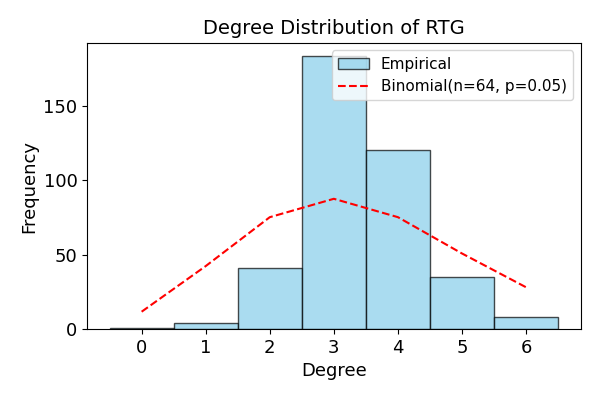}
        \caption{Width 32 (\(n=64\)), \(p=0.05\)}
        \label{fig:e6b}
    \end{subfigure}
    \hfill
    \begin{subfigure}[b]{0.70\linewidth}
        \includegraphics[width=\linewidth]{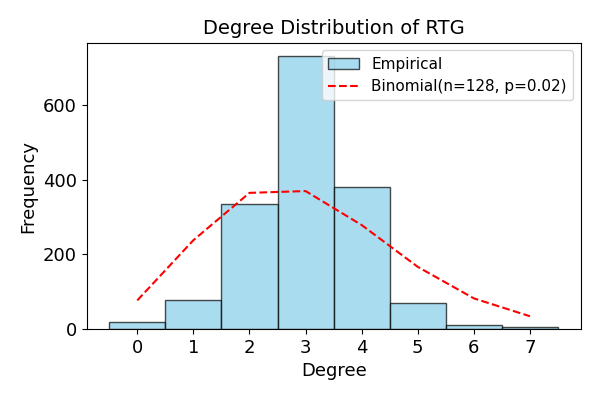}
        \caption{Width 64 (\(n=128\)), \(p=0.02\)}
        \label{fig:e6c}
    \end{subfigure}
    \caption{Experiment 4: Degree distributions for RTGs across increasing width. Each plot shows the empirical distribution over region degrees and the fitted Binomial curve.}
    \label{fig:degree-distributions}
\end{figure}

\subsection{Experiment 5: Generalization–Compression Duality}
We validate Theorem~\ref{thm:compression}, which predicts that the generalization error is upper-bounded by the average KL divergence across adjacent activation regions in the RTG. To stress-test this bound, we design an experiment that maximizes generalization gap by training a highly overparameterized network on randomly labeled data.

Specifically, we train ReLU MLPs with increasing depth \( d \in \{2,3,4,5,6\} \), fixed width \( m = 256 \). For each model:
\begin{itemize}
    \item We construct the RTG from activation patterns over the training inputs.
    \item For each RTG edge \( (R_i, R_j) \in E \), we compute:
    \[
    \mathrm{KL}(f|_{R_i} \;\|\; f|_{R_j}) = \mathrm{KL}\left(\mathbb{E}_{x \in R_i}[f(x)] \;\middle\|\; \mathbb{E}_{x \in R_j}[f(x)]\right),
    \]
    where \( f(x) \) denotes the softmax output of the network.
    \item We report the mean KL divergence across RTG edges and compare it to the observed generalization gap.
\end{itemize}

Results:
\begin{center}
\begin{tabular}{c|c|c}
\textbf{Depth} & \textbf{Gen Gap} & \textbf{Mean KL} \\
\hline
2 & 0.547 & 4.2305 \\
3 & 0.516 & 2.1211 \\
4 & 0.467 & 0.2223 \\
5 & 0.482 & 0.0789 \\
6 & 0.333 & 0.0020 \\
\end{tabular}
\end{center}

The results show that as width increases, the RTG becomes smoother (compressed), reducing functional variation between neighboring regions. This supports the dual view of generalization as compression across the graph, validating Theorem~\ref{thm:compression}

\section{Conclusion}
We have extended the ReLU Transition Graph (RTG) framework into a unified graph-theoretic model for understanding ReLU network behavior. By analyzing activation regions as nodes and single-neuron transitions as edges, we show that RTGs encode rich geometric and functional structure that governs expressivity and generalization.

Our theoretical results establish that RTGs exhibit expansion, binomial degree distributions, and spectral properties that control generalization error. We derive a region entropy bound on effective capacity and introduce a generalization–compression duality based on KL divergence across adjacent regions. These results connect classical ideas in spectral graph theory and information geometry to the internal structure of neural networks.

Empirical experiments on small MLPs validate these predictions. We construct RTGs explicitly and observe that region entropy saturates under overparameterization, spectral gap correlates with generalization, and KL divergence captures functional smoothness. These observations confirm that RTG-based metrics provide faithful indicators of learning dynamics.

Overall, our work reframes the study of ReLU networks in terms of discrete functional geometry, opening new directions for theory-guided architecture design, graph-structured regularization, and deeper understanding of how neural networks generalize.

\section{Broader Impact}
This work does not raise ethical or societal concerns. It advances the theoretical understanding of ReLU networks by framing them through the lens of discrete geometry. These insights may guide the development of more robust and interpretable models in broader machine learning systems.

\section*{Author Background Information}
Name: Sahil Rajesh Dhayalkar\\
Background: Sahil holds a Master's degree in Computer Science from Arizona State University and a Bachelor's degree in Computer Engineering from Veermata Jijabai Technological Institute (VJTI), Mumbai. He is based in San Diego, California. His research interests lie in robotics, artificial intelligence, deep learning, computer vision, and algorithms.\\
Email: \href{mailto:sdhayalk@asu.edu}{\texttt{sdhayalk@asu.edu}} \\


\begin{thebibliography}{00}
\bibitem{relu} Vinod Nair and Geoffrey E. Hinton. 2010. Rectified linear units improve restricted boltzmann machines. In Proceedings of the 27th International Conference on International Conference on Machine Learning (ICML'10). Omnipress, Madison, WI, USA, 807–814.
\bibitem{montufar2014number} G. Montúfar, R. Pascanu, K. Cho, and Y. Bengio, “On the number of linear regions of deep neural networks,” arXiv preprint arXiv:1402.1869, 2014.
\bibitem{hanin2019deep} B. Hanin and D. Rolnick, “Deep ReLU networks have surprisingly few activation patterns,” in *Advances in Neural Information Processing Systems (NeurIPS)*, vol. 32, 2019.
\bibitem{serra2018bounding} T. Serra, C. Tjandraatmadja, and S. Ramalingam, “Bounding and counting linear regions of deep neural networks,” in *Proc. 35th Int. Conf. Machine Learning (ICML)*, Stockholm, Sweden, Jul. 2018, pp. 4558–4566.
\bibitem{dhayalkar2025relu} Dhayalkar, Sahil. (2025). The Geometry of ReLU Networks through the ReLU Transition Graph. 10.48550/arXiv.2505.11692.
\bibitem{friedman1991second} J. Friedman, N. Linial, and M. Saks, “On the second eigenvalue of a graph and a network flow problem,” *Combinatorica*, vol. 11, no. 1, pp. 1–19, 1991.
\bibitem{hoory2006expander} S. Hoory, N. Linial, and A. Wigderson, “Expander graphs,” *Bull. Amer. Math. Soc.*, vol. 43, no. 4, pp. 439–561, 2006.
\bibitem{jacot2018neural} A. Jacot, F. Gabriel, and C. Hongler, “Neural tangent kernel: Convergence and generalization in neural networks,” arXiv preprint arXiv:1806.07572, 2020.
\bibitem{zaslavsky1975facing} Zaslavsky, Thomas. “Facing Up to Arrangements: Face-Count Formulas for Partitions of Space by Hyperplanes.” (1975).
\bibitem{arora2018understanding} R. Arora, A. Basu, P. Mianjy, and A. Mukherjee, “Understanding deep neural networks with rectified linear units,” arXiv preprint arXiv:1611.01491, 2018.
\bibitem{hanin2019complexity} B. Hanin and D. Rolnick, “Complexity of linear regions in deep networks,” in *Proc. 36th Int. Conf. Machine Learning (ICML)*, Long Beach, CA, USA, Jun. 2019, pp. 2596–2604.
\bibitem{raghu2017expressive} M. Raghu, B. Poole, J. Kleinberg, S. Ganguli, and J. Sohl-Dickstein, “On the expressive power of deep neural networks,” in *Proc. 34th Int. Conf. Machine Learning (ICML)*, Sydney, Australia, Aug. 2017, pp. 2847–2854.
\bibitem{telgarsky2016benefits} M. Telgarsky, “Benefits of depth in neural networks,” arXiv preprint arXiv:1602.04485, 2016.
\bibitem{montufar2017notes} G. Montúfar, “Notes on the number of linear regions of deep neural networks,” unpublished manuscript, Mar. 2017.
\bibitem{balestriero2021modern} Balestriero, R. (2021). Max-Affine Splines Insights into Deep Learning (Order No. 28736027). Available from ProQuest Dissertations \& Theses Global. (2572635695).
\bibitem{belkin2019reconciling} M. Belkin, D. Hsu, S. Ma, \& S. Mandal, Reconciling modern machine-learning practice and the classical bias–variance trade-off, Proc. Natl. Acad. Sci. U.S.A. 116 (32) 15849-15854, https://doi.org/10.1073/pnas.1903070116 (2019).
\bibitem{cohen2021gradient} J. M. Cohen, S. Kaur, Y. Li, J. Z. Kolter, and A. Talwalkar, “Gradient descent on neural networks typically occurs at the edge of stability,” arXiv preprint arXiv:2103.00065, 2022.
\bibitem{baraniuk2020spectrum} G. Ongie, A. Jalal, C. A. Metzler, R. G. Baraniuk, A. G. Dimakis and R. Willett, "Deep Learning Techniques for Inverse Problems in Imaging," in IEEE Journal on Selected Areas in Information Theory, vol. 1, no. 1, pp. 39-56, May 2020, doi: 10.1109/JSAIT.2020.2991563.
\bibitem{saxe} A. Saxe, Y. Bansal, J. Dapello, M. Advani, A. Kolchinsky, B. Tracey, and D. Cox, “On the information bottleneck theory of deep learning,” unpublished manuscript, Feb. 2018.
\bibitem{bianchini2014complexity} M. Bianchini and F. Scarselli, "On the Complexity of Neural Network Classifiers: A Comparison Between Shallow and Deep Architectures," in IEEE Transactions on Neural Networks and Learning Systems, vol. 25, no. 8, pp. 1553-1565, Aug. 2014, doi: 10.1109/TNNLS.2013.2293637.
\bibitem{lubotzky1988ramanujan} A. Lubotzky, R. Phillips, and P. Sarnak, “Ramanujan graphs,” *Combinatorica*, vol. 8, no. 3, pp. 261–277, 1988.
\bibitem{carlini2019neurons} N. Carlini, C. Liu, Ú. Erlingsson, J. Kos, and D. Song, “The secret sharer: Evaluating and testing unintended memorization in neural networks,” in *Proc. 28th USENIX Security Symp.*, Santa Clara, CA, USA, Aug. 2019, pp. 267–284. [Online]. Available: https://www.usenix.org/conference/usenixsecurity19/presentation/carlini
\bibitem{dhayalkarcombinatorial}
Dhayalkar, Sahil. (2025) “A Combinatorial Theory of Dropout: Subnetworks, Graph Geometry, and Generalization,” arXiv preprint arXiv:2504.14762 
\bibitem{hochreiter1997flat} S. Hochreiter and J. Schmidhuber, “Flat minima,” *Neural Computation*, vol. 9, no. 1, pp. 1–42, Jan. 1997.
\bibitem{neyshabur2017pac} B. Neyshabur, S. Bhojanapalli, and N. Srebro, “A PAC-Bayesian approach to spectrally-normalized margin bounds for neural networks,” arXiv preprint arXiv:1707.09564, 2018.
\bibitem{bartlett2017spectrally} P. Bartlett, D. J. Foster, and M. Telgarsky, “Spectrally-normalized margin bounds for neural networks,” arXiv preprint arXiv:1706.08498, 2017.
\bibitem{gouk2021regularisation} H. Gouk, T. M. Hospedales, and M. Pontil, “Distance-based regularisation of deep networks for fine-tuning,” arXiv preprint arXiv:2002.08253, 2021.
\bibitem{li2018visualizing} H. Li, Z. Xu, G. Taylor, C. Studer, and T. Goldstein, “Visualizing the loss landscape of neural nets,” arXiv preprint arXiv:1712.09913, 2018.
\bibitem{chaudhari2019entropy} P. Chaudhari, A. Choromanska, S. Soatto, Y. LeCun, C. Baldassi, C. Borgs, J. Chayes, L. Sagun, and R. Zecchina, “Entropy-SGD: Biasing gradient descent into wide valleys,” *J. Stat. Mech. Theory Exp.*, vol. 2019, no. 12, p. 124018, Dec. 2019.
\bibitem{xu2017information} Aolin Xu and Maxim Raginsky. Information-theoretic analysis of generalization capability of learning algorithms. In Advances in Neural Information Processing Systems, pages 2524–2533, 2017.
\bibitem{tishby2015deep} N. Tishby and N. Zaslavsky, “Deep learning and the information bottleneck principle,” arXiv preprint arXiv:1503.02406, 2015.
\bibitem{achille2018emergence} Alessandro Achille and Stefano Soatto.Emergence of invariance and disentanglement in deep representations.Journal of Machine Learning Research, 19(50):1–34, 2018.
\bibitem{baldi1989neural} P. Baldi and K. Hornik, “Neural networks and principal component analysis: Learning from examples without local minima,” *Neural Networks*, vol. 2, no. 1, pp. 53–58, 1989. 
\bibitem{poole2016exponential} B. Poole, S. Lahiri, M. Raghu, J. Sohl-Dickstein, and S. Ganguli, “Exponential expressivity in deep neural networks through transient chaos,” arXiv preprint arXiv:1606.05340, 2016.
\bibitem{neyshabur2018role} B. Neyshabur, Z. Li, S. Bhojanapalli, Y. LeCun, and N. Srebro, “Towards understanding the role of over-parametrization in generalization of neural networks,” arXiv preprint arXiv:1805.12076, 2018.
\bibitem{xu2019number} J. Xu and D. Hsu, “On the number of variables to use in principal component regression,” in *Advances in Neural Information Processing Systems (NeurIPS)*, vol. 32, 2019.
\bibitem{lee2019wide} R. Novak, L. Xiao, J. Hron, J. Lee, A. A. Alemi, J. Sohl-Dickstein, and S. S. Schoenholz, “Neural tangents: Fast and easy infinite neural networks in python,” arXiv preprint arXiv:1912.02803, 2019.
\bibitem{bronstein2017geometric} M. M. Bronstein, J. Bruna, Y. LeCun, A. Szlam, and P. Vandergheynst, “Geometric deep learning: Going beyond Euclidean data,” *IEEE Signal Process. Mag.*, vol. 34, no. 4, pp. 18–42, Jul. 2017.
\bibitem{kipf2017semi} T. N. Kipf and M. Welling, “Semi-supervised classification with graph convolutional networks,” arXiv preprint arXiv:1609.02907, 2017.
\bibitem{chen2021neural} Shaoyun Shi, Hanxiong Chen, Weizhi Ma, Jiaxin Mao, Min Zhang, Yongfeng Zhang. 2020. Neural Logic Reasoning. In Proceedings of the 29th ACM International Conference on Information and Knowledge Management (CIKM'20).
\bibitem{bastani2017interpreting} O. Bastani, C. Kim, and H. Bastani, “Interpreting blackbox models via model extraction,” arXiv preprint arXiv:1705.08504, 2019.
\bibitem{zhang2019interpretable} X. Zhang, N. Wang, H. Shen, S. Ji, X. Luo, and T. Wang, “Interpretable deep learning under fire,” arXiv preprint arXiv:1812.00891, 2019.
\bibitem{pytorch} A. Paszke, S. Gross, F. Massa, A. Lerer, J. Bradbury, G. Chanan, T. Killeen, Z. Lin, N. Gimelshein, L. Antiga, A. Desmaison, A. Kopf, E. Yang, Z. DeVito, M. Raison, A. Tejani, S. Chilamkurthy, B. Steiner, L. Fang, J. Bai, and S. Chintala,  “PyTorch: An imperative style, high-performance deep learning library,”  in Advances in Neural Information Processing Systems (NeurIPS), vol. 32, pp. 8024--8035, 2019


\end{thebibliography}
\end{document}